\documentclass{article}
\usepackage{kr}

\usepackage{times}
\usepackage{soul}
\usepackage{url}
\usepackage[hidelinks]{hyperref}
\usepackage[utf8]{inputenc}
\usepackage[small]{caption}
\usepackage{graphicx}
\usepackage{amsmath}
\usepackage{amsthm}
\usepackage{booktabs}
\usepackage{algorithm}
\usepackage{ifthen}
\usepackage[noend]{algpseudocode}
\usepackage{color}
\usepackage{tikz}
\usetikzlibrary{arrows.meta}

\newcounter{prithreshold}
\newcommand{\prio}[2]
{%
	\ifthenelse{\value{prithreshold} < #1}
	{%
		\textbf{[prio=#1]}
		#2
	}%
	{}%
}

\newtheorem{definition}{Definition}
\newtheorem{theorem}[definition]{Theorem}

\usepackage{xspace}

\newcommand{\sift}{\textnormal{\textsc{sift}}\xspace}
\newcommand{\siftp}{\textnormal{\textsc{sift+}}\xspace}
\newcommand{\synth}{\textnormal{\textsc{synth}}\xspace}
\newcommand{\synthp}{\textnormal{\textsc{synth+}}\xspace}
\newcommand{\localsynthp}{\textnormal{\textsc{local synth+}}\xspace}
\newcommand{\strips}{\textnormal{\textsc{strips}}\xspace}
\newcommand{\stripsp}{\textnormal{\textsc{strips+}}\xspace}

\newcommand{\Omit}[1]{}
\newcommand{\tup}[1]{\langle #1 \rangle}

\newcommand{\citeay}[1]{\citeauthor{#1} (\citeyear{#1})}

\DeclareMathOperator\seconds{s}
\DeclareMathOperator\atB{atB}
\DeclareMathOperator\atT{atT}
\DeclareMathOperator\leftof{leftof}
\DeclareMathOperator\belowof{belowof}
\DeclareMathOperator\hold{hold}
\DeclareMathOperator\holding{holding}
\DeclareMathOperator\adjacent{adjacent}
\DeclareMathOperator\atC{at}
\DeclareMathOperator\atP{atP}
\DeclareMathOperator\Max{\textit{max}}
\DeclareMathOperator\move{\textit{move}}

\DeclareMathOperator\moveto{\textit{moveto}}
\DeclareMathOperator\pick{\textit{pick}}
\DeclareMathOperator\drop{\textit{drop}}
\DeclareMathOperator\stack{\textit{stack}}
\DeclareMathOperator\unstack{\textit{unstack}}
\DeclareMathOperator\smove{\textit{m}}

\DeclareMathOperator\spick{\textit{p}}
\DeclareMathOperator\sdrop{\textit{d}}
\DeclareMathOperator\sstack{\textit{s}}
\DeclareMathOperator\sunstack{\textit{u}}

\DeclareMathOperator\umove{\textit{\underline{m}ove}}

\DeclareMathOperator\upick{\textit{\underline{p}ick}}
\DeclareMathOperator\udrop{\textit{\underline{d}rop}}
\DeclareMathOperator\ustack{\textit{\underline{s}tack}}
\DeclareMathOperator\uunstack{\textit{\underline{u}nstack}}
\DeclareMathOperator\init{\textit{Init}}
\DeclareMathOperator\Right{\textit{right}}
\DeclareMathOperator\Up{\textit{up}}

\DeclareMathOperator\Prob{\mathit{P}}
\DeclareMathOperator\Dom{\mathcal{D}}
\DeclareMathOperator\Pred{\mathcal{P}}
\DeclareMathOperator\A{\mathcal{A}}
\DeclareMathOperator\Prec{Pre}

\DeclareMathOperator\observation{{\omega}}
\title{Learning Lifted Action Models from Traces with\\Minimal
	Information About Actions and States}

\usepackage{graphicx}
\usepackage{subcaption}

\author{
    Jonas Gösgens$^1$\and
    Niklas Jansen$^1$\and
    Hector Geffner$^1$\\
    \affiliations
    $^1$RWTH Aachen University\\
    \emails
    \{jonas.goesgens, niklas.jansen, hector.geffner\}@ml.rwth-aachen.de
 }

\begin{document}

\maketitle

\begin{abstract}
It has been recently shown that lifted \strips models can be learned correctly and efficiently
from action traces alone; i.e.,  applicable action sequences from a hidden \strips model.
The result is remarkable because the states are not assumed to be observable at all, 
and yet it is not practical enough as \strips  actions include arguments that are not
needed for selecting the actions.  This shortcoming has  been addressed
by assuming that the action traces come instead from a hidden \stripsp model
where some action arguments are implicit in the hidden action preconditions.
A limitation of this approach, however, is that it assumes that
the states are fully observable. In this work, we relax these restrictions
and consider the problem of learning \stripsp  action domains from traces
in a more general context  where the traces carry partial information about
both actions and states. In particular, we formulate algorithms and completeness
results for three general cases, all of which assume full observability
of selected action arguments. In the first case, no  observability of the state is assumed; 
in the second case, full observability of some state predicates is assumed,
and in the third case, local observability of some state predicates is assumed instead.
Given a \stripsp  domain, these results characterize the conditions under which an equivalent domain
can be learned from traces.  Experimental results are reported.

\end{abstract}


\section{Introduction}

Recently, it  has been shown that lifted \strips models can be learned correctly and efficiently
from applicable action sequences  (action traces) drawn from a hidden \strips model,
via a simple algorithm called \sift \cite{sift}. The result, which builds on earlier work \cite{locm,locm2},
is particularly interesting as it implies that the domain  predicates can be learned accurately
and effectively from action traces alone,  without any restrictions,  and without assuming state
observability.\footnote{The only assumption is that the actions are well-formed,
meaning that they do not add atoms that are already true, or delete atoms that are already false.}

Still, the result is not practical enough, and it does not provide a \emph{lifted alternative
to model-based reinforcement learning}  approaches that are also aimed at learning dynamic
models from traces  \cite{sutton:book}. The reason is that
\strips action traces are not natural  and often presume knowledge of the hidden
domain that is to be learned. For example, in order to learn a domain like the sliding-tile
puzzles, 
the \strips actions   take as  arguments the tile to be moved,  its current position,
and the next position. Similarly, in the Blocksworld, an action like $unstack(x,y)$  has to mention
not only the block to be unstacked but also the block that is below it.
Many of these action arguments, however,  are not needed for selecting and
identifying the action to be done, but for modeling  the action in \strips.

The problem of learning action models in more expressive action languages
that do not require as many parameters has not received as much attention,
and a recent approach moves away from \strips to a simple  extension called
\stripsp,  where action arguments can be conveyed explicitly,  as usual,
or implicitly, as variables that bind to unique objects in the action preconditions \cite{synth}.
Indeed, the  first version of PDDL accommodated such \emph{implicit action arguments}
through the  \textsc{:vars} keyword  \cite{pddl}.
In \stripsp, the action of sliding  a tile to the left can be expressed indeed without any arguments.
The  precondition containing the lifted atoms $\atB(z_1)$, $\leftof(z_2,z_1)$, and $\atT(z_3,z_2)$,
encoding the position of the blank, the corresponding adjacency relation, and the tile to be moved and its position,
defines the action arguments implicitly by the unique objects that can bind to the $z_i$ variables in a state.

A limitation of the resulting algorithm,  called \synth  \cite{synth}, 
is that it assumes that the  states are \emph{fully observable}, which means that all
the predicates are given, and none has to be ``invented''.  In a way,
the \sift algorithm  ``invents'' predicates from full \strips action traces,
while   \synth deals with less informed \stripsp action traces but  assumes
that the predicates are given.  In this work, we build on these approaches and algorithms
to investigate what partial information is needed on both actions and states in the traces
in order to learn a domain.  In particular, we formulate algorithms and completeness
results for three general cases, all of which assume full observability
of the \stripsp actions in the traces.  In the first case, no  observability of the state
is assumed; in the second case, full observability of some state predicates is assumed; 
and in the third case, local observability of some state predicates is assumed instead.
Given a \stripsp  domain,  one can then determine exact conditions under which an equivalent
domain will be learned from traces. For example, the \stripsp domain for the
sliding-tile puzzle with no  action arguments  can be learned 
provided that the $\atB$ predicate encoding the blank position is \emph{fully observed},
and the  $\atT$ relation encoding the tile position is  \emph{locally observable.}
The resulting domain  learning algorithm, called \synthp, is an extension and
generalization of both \sift and \synth. 

The results are different from those that allow some observations in \strips traces
to be missing or corrupted by noise.
In our case, a non-observed predicate is never observed, and a locally observed predicate
defines precisely the observed atoms. From these crisp definitions,
we will  establish the scope of the model learning algorithms and establish
their correctness and completeness conditions.

The paper is organized as follows. We review  next related work and relevant background,
then we introduce the new notions, the new algorithms, and their properties. Some examples
are considered along the way and experiments are also  reported.

\section{Preview}

In the Delivery domain, there are packages spread in a grid that an agent must deliver to some target cell.
In \strips, the domain can be modeled with three action schemas

\begin{equation}
	\move(c,c'), \ \pick(o,c), \ \drop(o,c) \ ,
	\label{eq:1}
	\vspace*{.2cm}
\end{equation}

\noindent the first representing the   agent moving from a cell to an adjacent cell, the second, picking up an object from a cell, and the
third, dropping it on a cell. The \sift algorithm can learn an equivalent \strips model from action traces
made of ground instances of these three schemas. In this paper, we introduce algorithms and the corresponding theorems
for learning from a broader class of traces. For example, without assuming that the state is observable,
the \siftp algorithm will learn from \stripsp traces from instances of the actions
\begin{equation}
	\move(c'), \ \pick(o), \ \drop() \ ,
	\label{eq:2}
	\vspace*{.2cm}
\end{equation}

\noindent where some of the \strips action arguments that are not needed for selecting the actions
are omitted. This will be achieved by learning the ``mutex'' predicate $\atC(c)$ first, true for a single cell $c$ in any state,
encoding the agent position, and then learning from it, the ``mutex'' predicates $\atP(o,c)$, true for a single cell $c$ per object,
and $\hold(o)$, true for a single object $o$ at most.
At the same time, \synthp will learn from ground traces of  action schemas  like

\begin{equation}
	\Right(), \ \Up(), \ \ldots, \ \pick(o), \ \drop() \ .
	\label{eq:3}
	\vspace*{.2cm}
\end{equation}
\noindent provided that the predicate $at(c)$ is fully observed, and that the two directional adjacency relations
(e.g., $\leftof$ and $\belowof$) are  observed. The two  other relations needed,  $\atP(o,c)$ and $\hold(o)$,  will be learned.
Finally, in RL, it is  common to consider  actions with no arguments at all  like \cite{babyAI}
\begin{equation}
	\Right(), \ \Up(), \ \ldots, \ \pick(), \ \drop() \ .
	\label{eq:4}
	\vspace*{.2cm}
\end{equation}
\noindent The corresponding action models will  be learned in this case provided that, in addition, 
there cannot be more than one object in a cell, and that the predicate $\atP(o,c)$ is \emph{locally observed},
meaning that in this case only the  true $\atP(o,c)$ atoms for the cell $c$ for which $\atC(c)$ is true will be observed.
The  predicate $\hold(o)$ will   be learned.

While neither    \sift nor   \synth can learn the action models in the last three scenarios,  \siftp will learn the correct model in  the first of them, 
and  \synthp, in all of them.

\section{Related Work}

While many  languages have been developed for representing lifted dynamic models in logical form, the work on learning these models
has  been  focused mostly on \strips. 

\medskip

\noindent \textbf{Learning models  from actions.}
The \textsc{locm}  system \cite{locm1,locm}  accepts action traces as inputs,  and outputs lifted domain descriptions,
but it is a heuristic algorithm and its scope is not clear. The \sift algorithm   uses full  \strips action traces
and has been shown to be sound, complete, and scalable \cite{sift}.
A SAT-based  approach to lifted model learning has been developed as well, 
but the approach is not scalable \cite{bonet:ecai2020,ivan:kr2021}.

\medskip

\noindent \textbf{Learning models from states and actions.} The problem of learning lifted \strips models from state-action traces
has  received more attention  \cite{zhuo2013action,aineto2019learning,lamanna2021online,verma2021asking,macq:muise,le2024learning,bachor2024learning,xi2024neuro,aineto2024action}. While observability of the states can be  partial or  noisy, 
in almost all cases the observations reveal all the domain predicates and their arities, and the \strips actions reveal all the
arguments.  In contrast, \citeay{balyo:2024} introduce a SAT-based learning
formulation where only the action names are observed, 
while \citeay{paolo:2025} 
deals with  partially observable states and actions
although it is not complete, as it cannot  learn preconditions over action arguments
that are not observable and do not appear in the effects.  The \synth algorithm \cite{synth}, that   learns \strips models from \stripsp traces
combining actions and states, accounts  for such  missing action  arguments using  preconditions with free variables.

 \medskip

\noindent \textbf{Model-based RL.}   Model-based reinforcement learning algorithms learn controllers by also learning (stochastic) models, 
without making  assumptions about the structure of them \cite{sutton:book}. In the tabular setting, they  result in  flat state models
with  state transition probabilities obtained from simple counts  \cite{brafman:rmax}. In some cases, a first-order state language  is assumed but the state predicates
are given \cite{littman:rmax,leslie:probabilistic}. In more recent approaches, the learned dynamics is not represented compactly in languages such as \strips or
PDDL, but in terms of deep neural networks  \cite{fleuret:atari,dreamer:atari,timofte:atari}. A limitation of these methods, like other recent deep-learning approaches that learn \strips models from state images \cite{asai:latplan,asai:jair}, is that the learned action models are opaque and propositional.

\section{Background}

 We review \strips, \stripsp, and the input traces.

\subsection{\strips}

A classical STRIPS problem is a pair $\Prob=\tup{\Dom,I}$ where $\Dom$ is a first-order
\emph{domain} and $I$ contains information about the instance \cite{geffner:book,ghallab:book}.
The domain $\Dom=\tup{\Pred,\A}$ has a set $\Pred$ of predicate symbols $p$ and a set of action schemas $\A$
with preconditions and effects given in terms of  atoms $p(x_1, \ldots, x_k)$, where $p$ is a predicate
symbol of arity $k$, and each $x_i$ is an argument of the schema.
The instance information is a tuple $I=\tup{O, \init,G}$ where $O$ is a
set of object names $c_i$, and $\init$ and $G$ are sets of \emph{ground atoms} $p(c_1, \ldots, c_k)$
denoting the initial and goal situations. A STRIPS  problem $\Prob=\tup{\Dom,I}$ defines a state graph $G(\Prob)$ where nodes
represent reachable states in $\Prob$, the root node represents the initial state, and edges indicate state transitions labeled
with the actions causing them. A path in this graph represents an action sequence
that is applicable in the state represented by the first node.
\sift  learns domains  expressed in \strips  with \emph{negation} where negative literals
can be used in the initial situation,  action preconditions, and  goals.
The states in such a case are not sets of ground atoms but
sets of ground literals.  Since the goals of an instance $\Prob=\tup{\Dom,I}$  play no role in learning, 
we will regard $I$ as just representing the initial situation.

\subsection{\stripsp}

In order to learn from more natural traces than those resulting from \strips models,
the \synth algorithm learns from a slightly more expressive language called \stripsp.
The main difference between \stripsp and \strips is that the former accommodates
free $z_i$ variables in the action preconditions that must bind uniquely to single
objects in each state.
Formally, action schemas $a(x)$ in \stripsp have   (conjunctive)  preconditions  $\Prec(a(x))=\Phi(x,y,z)$
with free variables  among those of  $x$, $y$, and $z$ which are  pairwise disjoint sets of variables.
The $x_i$ and $z_i$  variables in $x$   and $z$ can appear in action effects, but the variables $y_i$ in $y$
cannot. A ground \stripsp action $a(o)$ is \emph{applicable}  in a state $s$ if its precondition
formula $\Phi(x,y,z)$ is \emph{satisfiable} in $s$ with a grounding that binds $x$ to $o$.
A grounding of   $\Phi(x,y,z)$ in $s$ is an assignment $\sigma$ of variables in the formula
to constants (objects) in the instance. A grounding \emph{satisfies}   $\Phi(x,y,z)$ in $s$ if the
resulting ground atoms are all true in $s$, and the  formula $\Phi(x,y,z)$ is \emph{satisfiable} in $s$
if some grounding of the variables satisfies it. The $z$ variables are \emph{determined}  by the variables
$x$ if the groundings that satisfy the formula must agree on the value (grounding) of $z$ when  they agree on the value of $x$.
The \emph{semantics} of the \stripsp action $a(x)$ with precondition $\Phi(x,y,z)$
is   the semantics of the \strips action $a'(x')$ that has the same preconditions and effects as $a(x)$
but with the $y$ and $z$ variables pushed as explicit arguments in $x'$.  Indeed, this is the way
to map \stripsp models into \strips models. 
As an example, the action $right()$ from  Delivery in (\ref{eq:3}) can be modelled with a precondition $\Phi(x,y,z): at(z_1) \land  rightof(z_1,z_2)$.
This precondition uniquely determines the implicit  arguments $z_1$ and $z_2$  of the action,
since the agent is located in exactly one cell and there is only one cell to its right.

\subsection{Traces}

An \emph{action trace} $a_0, \ldots, a_n$ in a domain instance $\Prob=\tup{\Dom,I}$ is an  action sequence that is applicable from a reachable state in $\Prob$.
An  action trace from $\Dom$ is a trace from an instance of $\Dom$. For each trace, there is  an   initial state $s_0$ and
states $s_{i+1}$ generated by the actions  in the  trace. The trace $s_0, a_0, s_1, a_1 , \ldots, a_{n-1},s_n$ is called a state-action trace. 
In this work  we consider action traces and partially observable state traces over (hidden) \stripsp domains of the form $\observation_0, a_0, \observation_1, a_1, \ldots$ where $\observation_i$ is a partial observation of the hidden state $s_i$.

When learning from action traces or partially observable traces, there is no assumption about whether any pair of hidden states $s_i$ and $s_j$
represents the same state or not. In certain cases, however, this  information is available (e.g., traces drawn from the same state) and
can be used. Action traces extended with such \emph{state equalities} are called
\emph{extended traces}.  It is useful to represent sets of extended traces $T$  as graphs $G_T$. In these graphs, the nodes represent the hidden states,
and two states  \emph{known} to represent the same state  can be  merged into a single node. The edges of the graph $G_T$ are labeled with the
actions mapping one state into the other.
We will refer to both  plain traces and extended traces as traces, and make their difference explicit when relevant.

\section{\sift: Learning from \strips Action Traces}

The \sift algorithm learns lifted  \strips models, including the domain predicates, from
\strips action traces, under the assumption that the hidden model is \emph{well-formed},
in the sense that it does not add atoms that are already true, nor deletes atoms that are already
false \cite{sift}. The same assumption  is made in \synth. The key idea in \sift is
that the domain predicates $p$ in a hidden \strips domain $\Dom$ can be represented in a suitable manner
as \emph{features}, and they can be recovered  from action traces $T$  alone in time that is linear in $|T|$.
For example, the atoms  $\atC(c)$ representing the position of the agent in a grid
are  affected only by the lifted actions  $\move(x,x')$ when $x'=c$ or $x=c$.
In the first case, $\atC(c)$ becomes true; in the second case, it becomes false.
The $\atC$ predicate is represented indeed by the feature $f_{\atC}=\tup{1,\{\move[1],\move[2]\}}$
of arity $1$ which takes one argument $c$, and which is affected only by  actions $\move$,
when $c$ is the first or second argument of $\move$. Otherwise, the atom $\atC(c)$
is not affected. More generally, a feature $f$ is a pair $f=\tup{k,B}$ where $k$ is the arity
of the feature, and $B$ is a set of action patterns $a[t]$ that affect $f$, where $a$
is an action name in $\Dom$ and  $t$ is a tuple $t=[t_1, \ldots, t_k]$ of $k$ indices $t_i$
that range over the argument indices of the action $a$. An action pattern $a[t]$ in $B$
says that the lifted   atom $f(x_{t_1}, \ldots, x_{t_k})$ is an effect of the
action $a(x_1, \ldots, x_{n_a})$. 

A feature $f=\tup{k,B}$  represents a hypothetical domain predicate and the ways in
which the actions affect it. The feature represents an actual domain predicate
if it is \emph{consistent} with the domain traces, and this consistency check
reduces to a fast 2-CNF consistency test \cite{sift}. The consistent 
features $f=\tup{k,B}$ are transformed  into consistent \emph{signed
features} $f'=\tup{k,A,D}$ where the set of action patterns $B$ is split
into add patterns in $A$, and  delete patterns in $D$, for $B=A\ \dot{\cup}\ D$.
\sift generates all possible action patterns $a[t]$ of arities $k=1, \ldots, \Max$,
$t=[t_1, \ldots, t_k]$, where $\Max$ is the max arity of an action in the traces,
and each of the possible features $f=\tup{k,B}$ is checked for consistency
individually.  The  features consistent with the traces encode the action effects
in a learned domain $\Dom_L$ whose preconditions can be inferred from the traces as well.
For a sufficiently rich set of traces $T$, the authors show that the learned domain $\Dom_T$
and the hidden domain $\Dom$ are equivalent.

\section{\synth:  \stripsp State-Action Traces}

\begin{algorithm}[t]
	\caption{Expands $Q^i(x,y,z^i)$ into $Q^{i+1}(x,y,z^{i+1})$ (Jansen, G\"osgens and Geffner 2025)}
	\label{alg:find_expansion}
	\begin{algorithmic}
		\State \textbf{Input:}  $Q(x,y,z^i) = \bigwedge_{j=1}^{i} Q_j(x,y,z^j)$ \Comment{Valid query Q}
		\State \textbf{Input:} AS $ = \{(a(o_i), s_i)\}_{i=1}^n$  \Comment{State-action pairs}
		\State \textbf{Output:} $Q_{i+1}(x,y,z^{i+1})$ \Comment{Valid extension of Q}
		\State
		\Function{Expand}{$Q(x,y,z^i), \text{AS}$}
		\State $Q \gets \{ p(w) \mid p \in P, w\in \{ x, y, z^{i+1} \}$, $z_{i+1} \in w$ $\}$
		\State $Q_0 \gets \{q_d\}$ \Comment{$q_d$ is dummy query $true$}
		\While{$Q_0 \not= \emptyset$}
		\State $Q_{next} \gets \{\}$
		\For{$q \in Q_0$}
		\For{$q' \in (Q \setminus q)$}
		\State $q'' \gets q \land q'$
		\State $\text{res} \gets \textsc{test*}(Q(x,y,z^i), q'', AS)$
		\If{$\text{res} = $ \emph{determined}}
		\State \Return $q''$
		\ElsIf{$\text{res} =$ \emph{valid}}
		\State $Q_{next} \gets Q_{next} \cup q''$
		\EndIf
		\EndFor
		\EndFor
		\State $Q_0  \gets Q_{next}$		
		\EndWhile
		\State \Return $Q^i(x,y,z^i)$  \Comment{It can't be extended further}
		\EndFunction
	\end{algorithmic}
\end{algorithm}

\synth learns from state-action traces drawn from a hidden \stripsp domain $\Dom$,
and for this, it assumes not only that the domain is well-formed and that
the action preconditions $\Phi(x,y,z)$ of action $a(x)$ determine uniquely
the value of the $z$-variables in any state $s$ where a ground instance  $a(o)$ applies
(a requirement in \stripsp), but also that these preconditions are easy to evaluate.
For this, each (existential)  $y_i$ variable in $y$ must occur  only once
in $\Phi(x,y,z)$, and the preconditions $\Phi(x,y,z)$  must be \emph{stratified}
in the following sense: they can be written as a sequence of $Q^i(x,y,z^i)$
expressions called \emph{subqueries}, each including the atoms that
include the variable $z_i$ and no variable $z_j$, $j > i$, $z=\tup{z_1, \ldots, z_n}$.
Moreover, not only must $Q^n(x,y,z^n)$ determine the value of the $z$ variables,
as demanded by \stripsp;  stratification demands that each subquery 
$Q^i(x,y,z^i)$ determines the value of each  $z^i$ variable \cite{synth}.

Under these conditions and with a suitable rich set of state-action traces $T$, \synth learns
a domain $\Dom_T$ that is  equivalent to the hidden domain $\Dom$. The key task  is learning the subqueries $Q^i(x,y,z^i)$, $i=1, \ldots, n$,
that set the values of the $z$-variables uniquely in each state where an instance $a(o)$
of an action $a(x)$ is applied. The procedure is shown in Algorithm~\ref{alg:find_expansion}.
The task is solved greedily, with no loss of completeness,
starting with $i=1$, initializing each subquery $Q^{i+1}(x,y,z^{i+1})$
to $Q^i(x,y,z^i)$, and  extending it  incrementally with lifted atoms $q_{i+1}$ that can involve
$x$ and $y$ variables, and must involve variable $z_{i+1}$ but no $z_j$ variable for $j > i+1$. 
After each extension, the satisfiability of $Q^{i+1}(x,y,z^{i+1})$ is checked in each state $s$ of  the traces where an action $a(o)$
is applied with $x=o$. If for each state $s$, the formula is  satisfiable (\emph{valid} in the code), $q_{i+1}$ is retained
in $Q^{i+1}(x,y,z^{i+1})$, else it is discarded. The step finishes when $Q^{i+1}(x,y,z^{i+1})$
determines the value of $z^{i+1}$ uniquely in all such states (marked as \emph{determined} in the code).
$Q^0(x,y,z^0)$ is empty, and the whole process finishes with $z=\tup{z_1, \ldots, z_n}$, when no other determined
variable $z_{n+1}$ with a denotation  different than the  $z_j$ variables found can be added to $z$.
Provided with the $z_i$ variables and the corresponding precondition queries
$Q^{i}(x,y,z^{i})$, the problem  reduces  to the well-known  problem of learning action
models from full \strips states and  actions, as the objects that bind to the $z_i$ variables
can then be regarded as explicit action arguments.




\subsection{\sift and \synth Revisited}

\sift and \synth  are presented as model learning algorithms: the first learns
\strips models from \strips action traces; the second learns \stripsp models from \stripsp state-action
traces. Yet another view of the core part of both  algorithms is possible. Given relational states $s$
and an action $a(x)$, \synth learns queries $q_i: Q^i(x,y,z^i)$ that bind the  $z_i$ variables to unique objects
in the preconditions of $a(x)$. These queries can be thought as \emph{referring expressions} with unique denotations.
\sift, on the other hand, {``invents''  predicates} (consistent with the traces) over the explicit  action arguments.
A key observation is that it is direct to use \sift  to ``invent'' predicates over implicit action arguments as well.
The only change that is needed is that in  the  action patterns $a[t]$ used to define the features,
the indices $t_i$ of $t$ must be allowed to point to an  explicit action argument $x_i$ or 
to an implicit action argument  $z_l$, whose value is captured by its query $q_l$ in the precondition of action $a$. 

Let us recall that \sift does not  learn the hidden domain $\Dom$
exactly, but with a sufficient rich set of traces, it learns  models $\Dom'$ that are \emph{equivalent} to  $\Dom$
in the following sense:

\begin{definition}
  A domain $\Dom'$ is equivalent to a domain $\Dom$ if positive action traces in $\Dom$ are positive action  traces in $\Dom'$, and negative action  traces in $\Dom$ are negative action
  traces in $\Dom'$, where
\begin{itemize}
\item An action trace $\tau$ is \emph{negative}  if $a_i$ is an action in $\tau$ with precondition $p$ such that
   an  action $a_j$  before $a_i$  deletes $p$, and no action between $a_j$ and $a_i$  adds $p$.
\item An action trace $\tau$ is \emph{positive} if it is not negative. 
\end{itemize}
\end{definition}

Namely, $\Dom$ and $\Dom'$ are equivalent if they accept and reject the same action traces; meaning that a feasible
plan in $\Dom$ must be a feasible  plan in $\Dom'$, and vice versa. This is indeed the definition that is used in \cite{sift} to validate the learned models.
The authors suggest an alternative definition that is  equivalent to this one, that says that $\Dom$ and $\Dom'$ are equivalent if they can be extended
into the so-called maximal domain descriptions $\Dom_{max}$ and $\Dom'_{max}$, compatible with the traces, such that $\Dom_{max}$ and $\Dom'_{max}$ are equal up to predicate
renaming. Note that \emph{static predicates}, namely, those which are not affected by the actions do not play any role in these
definitions as they are not a true property of the domain but of the instances \cite{sift}.




\begin{algorithm}[t]
  \caption{\synthp for learning \stripsp domains. It uses suitable extension of \sift to invent new predicates from explicit and implicit action arguments,
    and \synth\ to generate new queries from them and observable predicates if any}
	\label{fig:synthp}
	\begin{algorithmic}
		\State \textbf{Input:} Labeled graph $G=G_T$ of traces $T$
		\State \textbf{Input:} Obs. predicates $\Pred_{\observation}$ if any; else \synthp=\siftp
		\State \textbf{Output:} Learned \stripsp domain $\mathcal{D}_T$
		\State
		\Function{\synthp}{$G$,$\Pred_{\observation}$}
		\State $G' \leftarrow G$, $Q \leftarrow \emptyset$, $d \leftarrow 1$
		\While{$G'$  updated \textbf{and} $d < d_{\Max}$}
		\State $\Pred_n  \leftarrow$ \sift$(G')$ for normal features
		\State $\Pred_m \leftarrow$ \sift for mutex features$(\Pred_n, G')$.
		\State $\Pred \leftarrow \Pred_{\observation}\cup\Pred_n\cup\Pred_m$ 
		\State $Q \leftarrow Q \ \cup$ \synth$(G',\Pred)$.
		\State $G' \leftarrow$ $update(G',Q,\Pred)$ with new $z_i^a/Q^i$'s 
                \State $d \leftarrow d+1$
		\EndWhile
       		\State $\A \leftarrow $ Effects/precs  from learned features $\Pred_n$, $\Pred_m$
                \State $\A \leftarrow $ Add effects/precs over   from obs. predicates $\Pred_{\observation}$
		\State \Return $\Dom=\tup{\Pred,\A}$
		\EndFunction
	\end{algorithmic}
	\label{alg:synthp}
\end{algorithm}



\section{\siftp and \synthp: Pushing the Envelope}

A simple extension and combination of the ideas in  \sift and \synth  yields  a new   algorithm \synthp which  is more powerful than both.
\synthp learns from state-action traces assuming that a given  subset of predicates $\Pred' \subseteq \Pred$ is observable. If $\Pred'$
is empty, \synthp becomes an extension of \sift, which we call \siftp; if $\Pred'=\Pred$, it is equivalent to \synth. The interesting
and novel cases are thus  when $\Pred' \subset \Pred$. The pseudo-code for \synthp is shown in Alg.~\ref{fig:synthp} and will be explained below.
For the integration of \sift and \synth to be synergistic, however, a suitable extension of \sift is needed: the ability to handle a new type of feature that we call
\emph{mutex features}, as they will represent predicates $p$ of arity $k$ in which the last, $o_k$, argument is determined by the previous
ones $o_1, \ldots, o_{k-1}$. These features can be learned from action traces alone, very much as normal \sift features, with two differences spelled out below.

\subsection{Mutex Features}

A mutex feature $f=\tup{k,A,D}$ is a signed feature of arity $k$ where the action patterns $a[t]$ in $A$ are positive, i.e., add $f$-atoms,  and the
action patterns $b[t']$ in $D$ are negative, i.e., delete $f$-atoms. However, unlike the signed features that result from the consistency checks  in \sift
where the arity of the action patterns in both $A$ and $D$ is $k$; in  mutex features, the arity of the delete patterns $b[t']$ is $k-1$. The $k$-th argument
is implicit, under the assumption,  that has to be  verified, that the $k$-th argument of an $f$-atom is \emph{determined} by the other arguments.
Intuitively, while a \emph{consistent mutex feature}  $f=\tup{k,A,D}$ expresses that the atom $f(x_{t_1}, \ldots, x_{t_k})$ is a positive effect of the  action
$a(x)$ with $x=\tup{x_1, \ldots, x_n}$ for a positive pattern $a[t]$ in $A$; the atom $f(x'_{t'_1}, \ldots, x'_{t'_{k-1}},z_1)$ is a negative effect of the action
$b(x')$ with $x'=\tup{x'_1, \ldots, x'_m}$ for the negative pattern $b[t']$ in $D$. This effect makes use of a $z_1$ variable whose value is determined by the \stripsp
precondition $f(x_{t'_1}, \ldots, x_{t'_{k-1}},z_1)$.

\begin{definition}[Mutex Feature]
  A mutex feature $f=\tup{k,A,D}$ is a triplet given by a non-negative integer $k$, the arity of the feature,
  and  sets  of positive action patterns $a[t]$ in $A$ of arity $k$, and negative patterns $b[t']$ in $D$ of arity $k-1$. 
\end{definition}

For example, the mutex feature $f=\tup{1,A,D}$ with the positive action patterns of arity $1$, $A=\{\moveto[1]\}$,  and the negative action patterns of
arity $0$, $D=\{\moveto[]\}$, will capture the unary  predicate $\atC(x)$ which can only be true for a single constant $c$ in a state. Indeed, the feature expresses
that the action $\moveto(x)$ makes the atom $f(x)$ true and $f(z)$ false, provided the precondition $\atC(z)$ where $z$ is a determined variable.
If one can learn such features, then the action traces do not have to spell out as many arguments, and moreover, the resulting $z$ variables can be used to capture
implicit action arguments in other actions like $\pick(o)$ or $\drop()$, as we will see.

The key question is how the consistency of mutex features can be established given a set of extended traces. A slight extension of the ideas used in \sift
will provide the answer.

\medskip

Consider the graph $G_T$ associated with a set of extended traces $T$. The mutex feature $f=\tup{k,A,D}$ is consistent with the traces $T$ if two
conditions hold: 1)~every ground atom $f(o)$, $o=\tup{o_1, \ldots, o_k}$,  gets a single truth value in every node $n$ in $G_T$ following
the action patterns in $A$ and $D$, and 2)~for every pair of atoms $f(o)$ and $f(o')$ true in a node $n$, $o'=\tup{o'_1, \ldots, o'_k}$
with $o_k \not= o'_k$, it must be the case that $o_i \not= o'_i$ for some $i < k$. The second condition guarantees that in every node the value
of the argument at position $k$ of $f$ is determined by the precedent  arguments.\footnote{
A ground action $a(o)$ in  a node $n$ of an extended   trace makes a ground feature  $f(o')$ true in the following node $n'$,
for the mutex feature $f=\tup{k,A,D}$, if there is an action pattern $a[t]$ in $A$ such that $o'=t[o]$, where
$t[o]=\tup{o_{t_1},\ldots,o_{t_k}}$ for $t=[t_1,\ldots, t_k]$. From the well-formed assumption it follows that $f(o')$ must be false in the node $n$.
Likewise, a ground action $b(o)$  in  node $n$ makes the  ground feature $f(o')$ false in the following node $n'$,
if $f(o')$ is true in $n$, and there is an action pattern $b[t']$ in $D$ such that $t'[o]$ captures the first
$k-1$ arguments of $o'$.  Finally, if an edge from $n$ to $n'$ in $G_T$ is labeled with a ground action $c(o)$,
such that there are no $c[t]$ patterns in neither $A$ nor $D$,  the truth value of the ground atoms $f(o)$ in $n$
propagates to $n'$ and vice versa.} 

\begin{definition}[Consistent Mutex Feature]
  A mutex feature $f=\tup{k,A,D}$ is consistent over a set of extended traces $T$  if 1)~there is a consistent assignment of
  truth values to the ground atoms $f(o)$ affected by   ground actions in $T$ over the whole graph $G_T$,
  and 2)~for any two ground atoms $f(o)$ and $f(o')$ true in a node $n$ in $G_T$ with $o_k \neq o'_k$, it holds that $o_i \neq o'_i$ for some $i < k$. 
\end{definition}

Testing if a mutex feature $f=\tup{k,A,D}$ is consistent  with the traces $T$ can thus  be done efficiently
in time that is polynomial in the length of the traces as in \sift.

\subsection{Fully Observable Predicates}

\synthp learns  from \emph{partially observable state-action} traces of the form $\observation_0, a_1, \ldots, a_n,\observation_{n+1}$ where $a_i$ is a  \stripsp action,
and $\observation_i$ is a  \emph{partial  observation} of the hidden  state $s_i$ in the trace. We consider a crisp form of  partial state observability
where  a given  subset $\Pred^{\observation} \subseteq \Pred$ of predicate symbols $p$ in the domain $\Dom=\tup{\Pred,\A}$ is \emph{fully observable},
meaning that $\observation_i$ is given by all the $p$-atoms that are true in the hidden state $s_i$ for predicates $p \in \Pred^{\observation}$.

\begin{definition}[Full Predicate Observability]
  If a predicate $p \in \Pred$ in a \stripsp domain $\Dom=\tup{\Pred,\A}$ is \emph{fully observable}, then in a state $s$ of an instance of the domain,
  truth of all the ground $p$-atoms in $s$ is observed.
\end{definition}

Full predicate observability is thus different from full state observability as arises when there are no hidden domain predicates and they are all observable,
and also from a notion of local predicate observability that will be introduced below.

\subsection{\synthp and \siftp }

The \synthp algorithm is shown as  Algorithm~\ref{alg:synthp}. \synthp takes two arguments: the labeled  graph $G=G_T$ of the set of extended
traces over the hidden \stripsp domain, and the subset of fully observable predicates $\Pred^{\observation}$. If $\Pred^{\observation}$  is empty, \synthp learns solely
from action traces, and behaves like an extension of \sift, that we call \siftp. The nodes in the graph represent hidden states, and the edge labels
express  the ground actions that map one state into the next one. During the procedure,  other labels are added to the nodes and edges
like the truth values  of the learned and observed  $p$-atoms, and the new implicit action arguments $z_i^a$ and the queries $Q^i_a$ that
define their unique denotations for action $a$. In the inner loop, \synthp calls \sift iteratively  to invent new  predicates using action patterns defined
over the explicit action arguments, and the implicit action arguments found so far (initially none), and \synth  uses these predicates for defining
new implicit action arguments. The mutex and plain predicates invented by \sift correspond to the mutex and plain features that have been found
consistent with the traces in the graph. The consistent features, however, yield more than predicates: they also encode the action effects, through
the action patterns defining the features. The effects of the actions over the observed predicates in $\Pred^{\observation}$ are computed following
the simpler procedure in \synth, because such effects can be computed over state transitions and  not over  full trajectories. Yet,
\sift eventually ``rediscovers'' such predicates, as they also correspond to features consistent with the traces. The difference is that
observable predicates can be used from the beginning in query-preconditions to define implicit action arguments.
This difference will be explicit
below when analyzing the conditions on the hidden \stripsp domains under which \synthp and \siftp are complete.
The inner loop in \synthp can   reach a fixed point, but can also keep producing more implicit action arguments and more predicates.
This process can be stopped in two ways: using negative traces in training or using a bound $d_{\Max}$  on the number
of iterations. If $d_{\Max}$ is an upper bound on the length of the  longest paths in the \emph{dependency graph}  of the
the hidden domain (see below), the use of this bound does  not affect the completeness of \synthp.

\section{Properties}

The scope of \synthp, i.e., the class of \stripsp domains that it can learn correctly,
is given by the dependency graph of the domain and the set of predicates that are observable.

\subsection{\siftp Properties}

The dependency graphs have to be acyclic so that \synthp can learn the predicates one at a time.
In the absence of observable predicates, \synthp becomes \siftp, and the graph is defined as follows:

\begin{definition}
The directed  \emph{dependency graph}  $G_{\Dom}=(V,E)$ associated with a \stripsp domain $\Dom=\tup{\Pred,\A}$ with action schemas $\A$
is defined over vertices which represent the implicit  variables $z_i^a$, $i=0, \ldots, n_a$, for action schemas $a$, $a \in \A$,
and the predicates $p \in \Pred$. The edges of this graph are:

\begin{itemize}
	\item \textbf{Effects:} From  $p$ to $z_i^a$ if there is a $p$-effect of action $a$ that involves the variable $z_i^a$,
	\item \textbf{Preconditions:} From $z_i^a$ to $p$ if there is a $p$-atom in the subquery $Q^i(x,y,z^i)$ of action $a$
	\item  \textbf{Stratification:} From $z_j^a$ to $z_i^a$ if there is an precondition atom of $a$ that involves the variables $z_i^a$
	and $z_j^a$, $j > i$.
\end{itemize}
\end{definition}

A \stripsp  domain  $\Dom=\tup{\Pred,\A}$ is \emph{acyclic}  if the  graph $G_{\Dom}$ is acyclic; namely, it does
not contain a directed path from a node to itself.
The acyclic dependency graph of a couple of \stripsp domain encodings (Blocks and Delivery) is
shown in Figure~\ref{fig:dependency_all}.
If we let the \emph{rank} of an acyclic domain represent the length of the longest path in the dependency graph,
the conditions under which a \stripsp domain can be learned from action traces alone can be expressed as:
  
\begin{theorem}[\siftp]
  If the hidden  \stripsp domain $\Dom=\tup{\Pred,\A}$ is well-formed,  strongly  connected, and acyclic with rank bounded by $d_{\Max}$,
  then there is a set of extended traces  from which the algorithm \siftp will learn a domain equivalent to  $\Dom$.
\end{theorem}

\noindent In this definition, a  domain $\Dom$ is said to be  strongly  connected if in the instances of $\Dom$, the
states $s$ which are reachable from the initial state $s_0$, can reach $s_0$ back. The condition of acyclicity is needed
so that the predicates can be learned one at a time, while the strong connectedness ensures that there are traces
that can test the consistency of any feature.\footnote{These are all  sufficient conditions for the proof to hold,
not necessary conditions. In the experiments below, the algorithm learns domains 
that are  not  strongly connected, and $d_{\Max}$ can be set to infinity
as,   at some point,  no new predicates or implicit  arguments are found.}
The proof follows the one for the independent  theorem below.

\subsection{\synthp Properties}

The difference between \siftp  and \synthp is that the latter includes an initial non-empty subset of predicates $\Pred' \subseteq \Pred$
  that are fully observable. This extra knowledge extends the conditions under which a domain is learned, and this is  reflected
  in a  reduced dependency domain graph $G^{\observation}_{\Dom}$:
  
\begin{definition}
  The directed  dependency graph $G^{\observation}_{\Dom}=(V',E')$  associated with a \stripsp domain $\Dom=\tup{\Pred,\A}$ with predicates $\Pred$ and action schemas $\A$,
  when a subset $\Pred^{\observation} \subseteq \Pred$ of the predicates is observed, is the directed graph $G_{\Dom}=(V,E)$ for the domain $\Dom$, with the nodes
  for observed predicates  $p$ excluded. Namely, $V'$ is given by the $z_i^a$ variables for the action schemas $a \in \A$, and the non-observed
  predicates, and the edges $E'$  are those from $E$ involving such vertices only.
\end{definition}

Notice that if no predicate is observed, $\Pred'=\emptyset$, then $G^{\observation}_{\Dom}=G_{\Dom}$, while  if all predicates are observed,
$G_{\Dom}^{\observation} \not= G_{\Dom}$, and the \emph{acyclicity} of  $G_{\Dom}^{\observation}$ corresponds exactly to the conditions under which the \stripsp domain ${\Dom}$ is \emph{stratified.}
Indeed,  if a predicate $p$ in the  hidden domain $\Dom$ is fully observable, it can be used in preconditions to produce implicit action arguments
without having to be  learned. Actually, the effects on observable predicates can be determined later, without   inducing
dependencies in the graph $G^{\observation}_{\Dom}$.

\begin{theorem}[\synthp]
  Let $\Dom=\tup{\Pred,\A}$ be a well-formed and strongly connected  hidden \stripsp domain, and let $\Pred' \subseteq \Pred$
  be the subset of observable predicates. Then, if the resulting  dependency graph $G_{\Dom}^{\observation}$ is acyclic and with rank bounded by
  $d_{\Max}$,   the algorithm \synthp will learn a domain equivalent to $\Dom$.\label{thm:synthpcorrectandcomplete}
\end{theorem}

\begin{proof}[Proof-sketch]
Since the domain has an acyclic dependency graph, the queries can be learned one at a time.
As the hidden domain is assumed to be strongly connected, the trace can be extended such that, for every state and atom, there is a path reaching the state that adds or deletes the atom, and consequently, all atoms that may satisfy a query are known.
In iteration $i$, all queries of rank $i$ are learned, since their arguments are observed and the dependent predicates must be consistent with the traces.
After $d_{\Max}$ iterations, every query is recovered.
Because $d_{\Max}$ is finite and each iteration considers only finitely many queries, a finite set of traces suffices to rule out all unsatisfiable queries,
so the learned domain contains only queries and predicates consistent with the hidden domain.
\end{proof}

\subsection{Local Predicate Observability}

Full observability of a predicate $p$   assumes  that the learner has access to all the ground $p$-atoms  that  true in a state.
If the learner is the agent that is acting in the world, this is not a realistic assumption, and fortunately it is not necessary either.
For example, an agent moving in a grid may observe the true atoms $\atC(x)$ in the state,
but only the atoms $\adjacent(x,x')$  involving the single  value of $x$ for which $\atC(x)$ is true.
We refer to this alternative form of observability as \emph{local observability} and define it in a domain-independent way as follows:

\begin{definition}[Local objects and atoms]
  The set of \emph{local objects}  in a state $s$ over a \stripsp domain $\Dom=\tup{\Pred,\A}$ refers to the set of  objects appearing as explicit or implicit arguments of the
  ground actions \emph{applicable} in $s$. The \emph{local atoms} in $s$ are the ground atoms that involve a local object.
\end{definition}

\noindent While a \emph{fully-observable} predicate $p$, reveals the set of all true  $p$-atoms  in  a state $s$, a \emph{locally observable} predicate $p$  reveals
the set of all true $p$-atoms  which  are local in $s$:

\begin{definition}[Local Observability]
  If a predicate $p \in \Pred$ in a \stripsp domain $\Dom=\tup{\Pred,\A}$ is \emph{locally observable}, then in a state $s$ of an instance of the domain,
   the   truth of all the ground $p$-atoms which are local in $s$ is known.\label{def:localobspred}
\end{definition}

If $p(c,d)$ is a ground atom in an instance of a domain $\Dom$ where $p$ is locally observable, then the truth of $p(c,d)$ is assumed to be known in a true hidden state $s$ iff it is a local atom in $s$. If $p$ is a fully observed predicate, on the other hand, there is no distinction between local and non-local $p$-atoms.
The interesting point is that \synthp  preserves completeness even when some of the fully observed predicates in $\Pred^{\observation} \subseteq \Pred$ become \emph{locally observable}:

\begin{definition}
  A \emph{key} predicate $p$ in a \stripsp domain $\Dom=\tup{\Pred,\A}$ is a predicate that appears in an atom of a lifted action precondition and which involves
  no explicit action arguments $x_i$, and a single implicit action argument $z_i$.
\end{definition}

For example, in the \emph{n-puzzle}, the action \emph{Right} can  have  no explicit argument  in a  \stripsp encoding with
preconditions $\atB(z_1)$, $\Right(z_2,z_1)$, $\atT(z_3,z_2)$. In such a domain, the predicate $\atB$ is a key predicate as the precondition atom  $\atB(z_1)$
involves no explicit argument and a single $z$-argument. The  predicates $\Right$ and $\atT$ are no key predicates as they involve a pair of $z$-variables.

We call the resulting algorithm  \localsynthp. \localsynthp differs from \synthp in two minor ways. When testing the \emph{validity} of a precondition and
the \emph{determination} of a query $Q^{i+1}(x,y,z^{i+1})$ within the \textsc{test*} procedure shown in the query-expansion procedure in Alg.~\ref{alg:find_expansion},
special care is taken of atoms $p(o)$ for locally observable predicates $p$ that are not local in a state. The truth of such atoms is  unknown in $s$,
yet they can be assumed to be \emph{false} for  testing  validity of a precondition, and \emph{true} for testing whether a query ensures that $z^{i+1}$ binds to a unique
object  in $s$.

\begin{theorem}[\localsynthp]
  Let $\Dom=\tup{\Pred,\A}$ be a well-formed and strongly connected  hidden \stripsp domain, and let $\Pred^{\observation} \subseteq \Pred$
  be a  subset of predicates such that the \emph{key predicates} in $\Pred^{\observation}$ are fully observable, and the others are \emph{locally observable}.
  Then, if the  dependency graph $G_{\Dom}^{\observation}$ formed by  assuming  that all predicates $\Pred^{\observation}$ are fully observable, is acyclic and of rank bounded by $d_{\Max}$,
  the algorithm \localsynthp  will learn a domain equivalent to  $\Dom$.
\end{theorem}

The theorem is a result of the definition of local observability and the change to the \textsc{test*} procedure, which ensures that non-observed atoms in the state do not affect the relevant queries or introduce additional invalid queries.


\section{Examples}

We look at how \siftp and \synthp learn some  domains.

\subsection{Examples: \siftp}

We consider  two domains which are fully learned without observing any predicate, whose  dependency graphs are shown in Figure~\ref{fig:dependency_all}.

\begin{figure}[t]
	\centering
	
	\begin{subfigure}[t]{0.48\columnwidth}
		\centering
		\begin{tikzpicture}[
			scale=0.9,
			node distance=1.5cm,
			every node/.style={
				draw,
				minimum size=.6cm,
				inner sep=0pt
			}
			]
			
			\node[rectangle, draw, rounded corners] (A) at (0,0) {$f_1$};
			\node[circle, draw] (B)  at (1,-.5) {$z_1^{\sdrop}$};
			\node[circle, draw] (C)  at (1,.5) {$z_1^{\sstack}$};
			\node[rectangle, draw, rounded corners] (D)  at (2,.5) {$f_2$};		
			\node[circle, draw] (E)  at (3,0.5) {$z_1^{\sunstack}$};
			
			\draw[-Stealth] (B) -- (A);
			\draw[-Stealth] (C) -- (A);
			\draw[-Stealth] (D) -- (C);
			\draw[-Stealth] (E) -- (D);
		\end{tikzpicture}
		\caption{\emph{Blocksworld}}
		\label{fig:dependency_blocks}
	\end{subfigure}
	\hfill
	\begin{subfigure}[t]{0.48\columnwidth}
		\centering
		\begin{tikzpicture}[
			scale=0.9,
			node distance=1.5cm,
			every node/.style={
				draw,
				minimum size=.6cm,
				inner sep=0pt
			},znode/.style={
				circle, draw
			},fnode/.style={
				rectangle, draw, rounded corners
			}
			]
			
			\node[fnode] (B) at (1,1) {$f'_1$};
			\node[fnode] (A) at (2,1) {$f'_2$};		
			\node[znode] (C) at (3,0) {$z_1^{\sdrop}$};
			\node[znode] (D) at (1,0) {$z_1^{\smove}$};
			\node[znode] (E) at (0,0) {$z_1^{\spick}$};
			\node[znode] (F) at (2,0){$z_2^{\sdrop}$};
			
			\draw[-Stealth] (C) -- (A);
			\draw[-Stealth] (D) -- (B);
			\draw[-Stealth] (E) -- (B);
			\draw[-Stealth] (F) -- (C);
			\draw[-Stealth] (F) -- (B);
		\end{tikzpicture}
		
		\caption{\emph{Delivery}}
		\label{fig:dependency_delivery}
	\end{subfigure}
	
	\caption{Dependency graphs for two domains}
	\label{fig:dependency_all}
	
\end{figure}
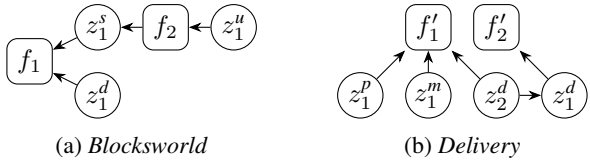

\medskip

\noindent \textbf{Blocksworld:}
The \strips actions $\ustack(x_1,x_2)$, $\uunstack(x_1,x_2)$, $\upick(x_1)$, and $\udrop(x_1)$
are  reduced automatically  to \stripsp actions of lower arity: $\stack(x_2^s)$, $\unstack(x_1^u)$, $\pick(x_1^p)$, $\drop()$.
All the action arguments dropped  are  recovered by \siftp,  without observing the  states,  using mutex features.
The acyclic dependency between the features and $z$-variables can be seen in Figure \ref{fig:dependency_blocks}.
First, the mutex feature
$$f_{1}=\tup{1,\{\unstack[1], \pick[1]\},\{\stack[\ ], \drop[\ ]\}} \text{,}$$ is learned, which captures the predicate $\holding$. Using preconditions $Q^{\spick}_1 = f_1(z_1^{\sdrop})$ and $Q^{\sstack}_1 = f_1(z_1^{\sstack})$, the block that is currently held can be inferred, resulting in actions $\drop(z_1^{\sdrop}), \stack(x_2^{\sstack},z_1^{\sstack})$.
Only afterward, the mutex feature
$$f_{2}=\tup{2,\{\stack[2,1]\},\{\unstack[1]\}}$$
is learned, which captures the predicate $on$. Using $Q^{\sunstack}_1 = f_2(x_1^{\sunstack}, z_1^{\sunstack})$, the block below $z_1^{\sunstack}$ is inferred, resulting in action $\unstack(x_1^{\sunstack},  z_1^{\sunstack})$.
After this, no further arguments can be derived using referring expressions over mutex features.

\medskip

\noindent \textbf{Delivery:}
Consider \emph{delivery}, a domain with multiple agents and packages located on a grid of cells $c$.
In \strips, this domain can be described using an action
$\umove(x_1^{\smove}, x_2^{\smove}, x_3^{\smove})$, where agent $x_1^{\smove}$ moves from $x_2^{\smove}$ to $x_3^{\smove}$, and actions
$\upick(x_1^{\spick}, x_2^{\spick}, x_3^{\spick})$, where agent $x_1^{\spick}$ picks package $x_2^{\spick}$ from cell $x_3^{\spick}$, and similarly
$\udrop(x_1^{\sdrop}, x_2^{\sdrop}, x_3^{\sdrop})$.
In \stripsp, the explicit actions can be reduced to
$\move(x_1^{\smove}, x_3^{\smove})$, $\pick(x_1^{\spick}, x_2^{\spick})$, and $\drop(x_2^{\sdrop})$.
All implicit arguments can be recovered using independent mutex features, as shown in the dependency graph in Figure~\ref{fig:dependency_delivery}.
Using these action schemas, the mutex feature
\[
f_1' = \tup{2,\{\move[1,2]\},\{\move[1]\}}
\]
can be learned, capturing the \strips predicate $at(x_1,x_2)$.
With the preconditions
$Q_{\smove}^1 = f_1'(x_1^{\smove}, z_1^{\smove})$ and
$Q_{\spick}^1 = f_1'(x_1^{\spick}, z_1^{\spick})$,
additional arguments can be inferred, yielding the actions
$\move(x_1^{\smove}, x_3^{\smove}, z_1^{\smove})$ and
$\pick(x_1^{\spick}, x_2^{\spick}, z_1^{\spick})$,
which now include the current position of the agent.
Even without exploiting these newly inferred arguments, the mutex feature
\[
f_2' = \tup{2,\{\pick[2,1]\},\{\drop[1]\}}
\]
can be learned, capturing the \strips predicate $\holding(p,a)$.
Using the precondition
$Q^{\sdrop}_1 = f_2'(x_2^{\sdrop}, z_1^{\sdrop})$,
we obtain the action $\drop(x_2^{\sdrop}, z_1^{\sdrop})$, where the agent can be inferred from the package it is holding.
Only after the agent that is holding the package is known, its position can be recovered using the precondition
$Q_{\sdrop}^2 = f_1'(z_1^{\sdrop}, z_2^{\sdrop})$,
yielding the full action
$\drop(x_2^{\sdrop}, z_1^{\sdrop}, z_2^{\sdrop})$.
At this point, no further preconditions over mutex features yield new arguments, and the resulting action schemas are equivalent to the original \strips action schemas.


\begin{table*}[t]\setlength{\tabcolsep}{1.9pt}\centering
	\begin{tabular}{lccccccc|ccrcrcr|crrr|cc}
		\multicolumn{8}{c}{Data} & \multicolumn{7}{|c}{\siftp (top) / \synthp (bottom)} & \multicolumn{4}{|c}{Verification} & \multicolumn{2}{|c}{\synth} \\
		Domain      &$\Pred^f$&$\Pred^s$& $\textrm{O}$ & $\textrm{L}$  & $|x'|$ &$\Pred_{\observation}^f$&$\Pred_{\observation}^s$& $|z\!\cap\! x'|$ & $|z \!\setminus\! x'|$ &$F^m$&$F^m_c$&$F$&$F_c$&$\textrm{T}_\textrm{L}$& $\textrm{O}_\textrm{V}$ & $\textrm{S}_\textrm{V}$\ & $\textrm{T}_\textrm{V}$ & $\%\textrm{V}$ & $|z\!\cap\! x'|$ & $|z \!\setminus\! x'|$ \\
		\midrule
		blocks3     &$3$&$1$&$  6$&$  1k$&$ 7$&$0$&$0$&$ 2$&$ 0$&$     66$&$  2$&$    5$&$  4$&$    8\seconds$&$  7$&$ 1k$&$  101\seconds$&$100\%$&$2$           &$                     0$ \\
		blocks4     &$5$&$0$&$  7$&$  1k$&$ 6$&$0$&$0$&$ 3$&$ 0$&$     21$&$  3$&$   12$&$  8$&$    5\seconds$&$  8$&$ 1k$&$   68\seconds$&$100\%$&$3$           &$                     0$ \\
		delivery    &$3$&$1$&$ 13$&$  7k$&$ 9$&$0$&$0$&$ 4$&$ 0$&$    101$&$  5$&$    7$&$  5$&$   22\seconds$&$ 14$&$ 5k$&$  313\seconds$&$100\%$&$4$           &$                     0$ \\
		driverlog   &$4$&$2$&$ 11$&$ 61k$&$19$&$0$&$0$&$ 9$&$ 2$&$1505800$&$ 16$&$   49$&$ 23$&$111620\seconds$&$13$&$41k$&$20198\seconds$&$100\%$&$9$           &$                     0$ \\
		ferry       &$4$&$1$&$ 10$&$  2k$&$ 6$&$0$&$0$&$ 4$&$ 0$&$     27$&$  3$&$    6$&$  4$&$    4\seconds$&$ 11$&$ 2k$&$   64\seconds$&$100\%$&$4$           &$                     0$ \\
		grid        &$6$&$6$&$ 13$&$ 13k$&$13$&$0$&$0$&$ 8$&$ 0$&$    769$&$  7$&$   11$&$  7$&$  109\seconds$&$ 17$&$ 9k$&$ 1415\seconds$&$100\%$&$9$           &$                     4$ \\
		gripper     &$4$&$4$&$ 11$&$  2k$&$ 8$&$0$&$0$&$ 4$&$ 2$&$    414$&$ 28$&$   14$&$ 14$&$   18\seconds$&$ 12$&$ 2k$&$  203\seconds$&$100\%$&$5$           &$                     3$ \\
		gripper4    &$4$&$4$&$ 13$&$  2k$&$ 8$&$0$&$0$&$ 4$&$ 0$&$     66$&$  5$&$    5$&$  5$&$    9\seconds$&$ 14$&$ 2k$&$  113\seconds$&$100\%$&$4$           &$                     0$ \\
		hanoi       &$2$&$1$&$  6$&$  1k$&$ 3$&$0$&$0$&$ 1$&$ 0$&$     29$&$  2$&$    3$&$  3$&$    8\seconds$&$  7$&$ 1k$&$  140\seconds$&$100\%$&$1$           &$                     0$ \\
		logistics   &$2$&$8$&$ 14$&$ 23k$&$13$&$0$&$0$&$ 5$&$ 0$&$    561$&$  9$&$    7$&$  7$&$   98\seconds$&$ 19$&$21k$&$10028\seconds$&$100\%$&$6$           &$                     7$ \\
		miconic     &$3$&$3$&$  7$&$  1k$&$ 8$&$0$&$0$&$ 4$&$ 2$&$    102$&$  9$&$    8$&$  8$&$    2\seconds$&$ 11$&$ 1k$&$   55\seconds$&$100\%$&$6$           &$                     0$ \\
		$n$-puzzle  &$2$&$3$&$ 10$&$  6k$&$16$&$0$&$0$&$ 8$&$ 0$&$   5006$&$ 41$&$   36$&$ 22$&$  152\seconds$&$ 23$&$10k$&$ 6082\seconds$&$100\%$&$16$          &$                     0$ \\
		$c$-puzzle  &$2$&$4$&$ 17$&$  1k$&$12$&$0$&$0$&$ 4$&$ 0$&$    791$&$  3$&$   16$&$  2$&$   39\seconds$&$ 31$&$ 1k$&$ 1153\seconds$&$100\%$&$12$          &$                     0$ \\
		sokoban     &$2$&$2$&$ 15$&$  5k$&$ 5$&$0$&$0$&$ 2$&$ 0$&$     61$&$  1$&$    7$&$  3$&$  148\seconds$&$ 23$&$ 5k$&$ 3803\seconds$&$100\%$&$3$           &$                     0$ \\
		\midrule\midrule
		sokoban     &$2$&$2$&$ 15$&$  6k$&$ 5$&$0$&$2$&$ 3$&$ 0$&$     61$&$  1$&$    7$&$  3$&$  453\seconds$&$ 24$&$ 5k$&$12331\seconds$&$100\%$&$3$           &$                     0$ \\
		logistics*  &$2$&$8$&$ 14$&$ 25k$&$13$&$0$&$1$&$ 6$&$ 4$&$   3833$&$ 16$&$   11$&$ 11$&$ 1385\seconds$&$ 18$&$21k$&$56344\seconds$&$100\%$&$6$           &$                     4$ \\
		s-delivery  &$4$&$2$&$ 19$&$  7k$&$12$&$1$&$2$&$11$&$ 0$&$    137$&$  3$&$    6$&$  4$&$  492\seconds$&$ 20$&$ 5k$&$12538\seconds$&$100\%$&$11$          &$                     0$\\

	\end{tabular}
	\caption{Results table.
		Table of results when learning a \stripsp domain $\mathcal{D}$ from \emph{extended} traces of size $\textrm{L}$ from an instance with $\textrm{O}$ objects of a \strips domain $\mathcal{D'}$ containing $|x'|$ action arguments, $\Pred^f$ \emph{non-static} and $\Pred^s$ \emph{static} predicates.
		Hereby $\Pred_{\observation}^f$ denotes the number of fluent predicates and $\Pred_{\observation}^s$ the number of static predicates observable to \synthp, always zero for the \siftp cases.
		$|z\!\cap\! x'|$ denotes the number of \strips arguments not explicit in the trace but recovered by the algorithms, and $|z \!\setminus\! x'|$ the number of additional learned arguments not contained in $\mathcal{D'}$, with \siftp / \synthp columns showing our results and \synth columns showing the corresponding numbers reported for \synth \protect\cite{synth}.
		$T_L$ is the learning time.
		$F$ plain features and $F^m$ mutex features were generated and tested for consistency, with $F_c$ and $F^m_c$ consistent features, respectively.
		Each learned domain was validated using 24 positive and 24 negative traces each of length $\textrm{S}_\textrm{V}$ sampled from instances containing $\#\textrm{O}_\textrm{V}$ objects. $\textrm{T}_\textrm{V}$ is the verification time, and $\%\textrm{V}$ the success rate.}
	\label{table:results}
\end{table*}

\subsection{Examples: \synthp}
The next examples illustrate that  \synthp strictly extends both  \synth and \siftp.

\medskip

\noindent \textbf{Delivery:} Consider the delivery definition from Equation~(\ref{eq:3}) with a single agent, where the agent's position ($at(x)$) and the adjacency relations ($leftof(x,y)/belowof(x,y)$) are observed, and only the action $pick(x_1)$ has an observed argument, namely the package that is picked.
The agent's position can be inferred for each action $a$ using the query $Q_a^1 = at(z_1)$, and for the move actions the next position can be inferred using the adjacency relations (e.g. $Q^2_{down} = belowof(z_1, z_2)$ for action $down$).
The mutex feature \[f_1 = \tup{1,\{pick[1]\},\{drop[]\}}\] infers the package currently held, and the query $Q^2_{drop} = f_1(z_2)$ then infers the package being dropped.
These queries recover all arguments in the hidden domain, so any fluent predicate \sift could learn from full \strips actions can now be learned. Since all fluent predicates are learned and all static predicates are observed, \synthp can infer every argument that \synth can infer from full states. By contrast, \synth cannot infer the package that is dropped since the predicate is missing and \siftp cannot infer the agent's position, so neither learns the correct domain.

\medskip

\noindent \textbf{Sokoban:} Let the predicates $connected$, $connected_2$ be observed, as well as actions $up(x_1), push(x_1)$, where $x_1$ denotes the next position of the agent.
The mutex feature \[f_1 = \tup{1,\{move[1], push[1]\},\{move[], push[]\}}\] infers the current position of the agent, and using the query $Q^1_{move} = Q^1_{push} = f_1(z_1)$ for both actions this position can be inferred.
Afterward, the query $Q^2_{move}=connected(x_1,z_2), connected_2(z_1,z_2)$ infers the cell to which a box is pushed.
Again all arguments are recovered and all predicates and queries can be learned. The examples are beyond the scope of both  \siftp and \synth.


\section{Experiments}
We evaluate \siftp and \synthp on a set of \strips domains under different observability assumptions about the actions and/or states in the traces. 



\medskip

\noindent \textbf{Experimental Setup:}
For each \strips domain, a corresponding \stripsp domain was constructed by removing arguments (and predicates) such that the algorithms could infer all missing arguments.
As input, six \emph{extended} traces over the same instance of these domains were generated.
The only exceptions are \textit{n-puzzle}, \textit{logistics} and \textit{logistics*}, for which traces of different instances were used.
We compare our results with those reported for \synth \cite{synth}, and for additional experiments not included in the paper, we used their publicly available implementation \cite{synth_implementation}.
All experiments were repeated 25 times and run on 10 CPU cores with a clock rate ranging from 2.1GHz to 2.9GHz,  using up to 450GB of memory.
All data and the implementation used are publicly available \cite{synthp_implementation}.

\medskip

\noindent \textbf{Domains:}
We evaluated \siftp on the same set of \emph{well-formed} planning domains used in \cite{sift}, where 
$n$-puzzle and $c$-puzzle refer to the sliding-tile puzzle with locations defined as coordinates and cells, respectively \cite{synth}. We additionally included a \emph{gripper} instance with four rooms and three grippers (\emph{gripper4}). \synthp was evaluated on three domains: \emph{sokoban}, the \emph{delivery} domain from Equation~\ref{eq:3}, and \emph{logistics*}, which is equivalent to \emph{logistics} except that cities can have more than one airport.

\medskip

\noindent \textbf{Translation to \stripsp :}
The \stripsp domains, from which action traces were generated to test \siftp, were generated automatically from the \strips domains
\cite{synth}.
For testing \synthp, some \stripsp action arguments were rendered unobservable, and some state predicates were chosen
to be observable instead.

\medskip

\noindent \textbf{Data:}
For each \stripsp domain $\Dom$, six \emph{extended} traces were sampled using breadth-first search, starting from a random reachable
state of an instance. For all domains but \textit{logistics}, \textit{logistics*}, and the  \textit{n-puzzle},  a single instance was used to sample the traces.
For efficiency reasons, in these three domains, multiple smaller instances were used instead.

\medskip

\noindent \textbf{Verification:}
For verification, we generate both positive and negative extended traces from an instance of the \strips domain that is different from the one used to
learn $\Dom$. Negative traces are traces in which the last action is not applicable because a precondition on a fluent predicate is false in the corresponding state.
Positive traces are traces in which all actions are applicable. Verification fails if in the learned model, a negative trace is found to be positive,
or a positive trace is found to be negative. A verification rate of 100\% means that all traces are classified correctly.

\medskip

\noindent \textbf{Analysis of \siftp:}
The results in the upper part of Table~\ref{table:results} show that \siftp is able to learn correct domains ($100\%$ verification) in all domains from action traces alone, even though these traces convey, on average, only $55\%$ of the action arguments contained in the original \strips domains.
This can be seen in the table where column $|x'|$ shows the number of action arguments in the \strips domain and $|z\cap x'|$ shows the number of arguments ``recovered'' by \siftp. For example, in \emph{ferry}, the traces recover 4 action arguments compared to 6 in the \strips domain, meaning that only 2 action arguments are explicit in the \stripsp domain. Similarly, in \emph{grid}, 8 arguments are recovered even though only 5 are provided explicitly in the traces. In both cases, \siftp successfully learns the correct domain.
In some domains, \synth recovers significantly more arguments than \siftp, for example in \emph{c-puzzle}, where \synth can learn from action traces with actions with no arguments, but uses state observability.

\medskip

\noindent \textbf{Analysis of \synthp:}
The bottom  part of Table~\ref{table:results} shows that \synthp learns a correct domain in every case, using the same number of explicit action arguments as \synth, i.e.,
those in the \stripsp encoding, but observing significantly fewer predicates.

In \emph{sokoban}, only the static predicates defining the grid are observed, and no  fluent predicate is observed.
Similarly, in \emph{logistics*}, only the static predicate encoding the relations between cities and locations is observed.
In \emph{delivery}, as in (\ref{eq:3}) above, only the fluent predicate $at(x_1)$ is  observed  along with  the static grid predicates,
whereas \synth uses all six fluent predicates.
In all these  cases, neither \synth nor \siftp can learn the domains provided with  the same inputs.



\section{Conclusions}

For  action model learning to become a practical, lifted alternative to model-based reinforcement learning,
it must be able to use natural traces  conveying  minimal information about actions and states, revealing
certain predicates only and no more action arguments than those which are necessary for selecting the actions.
Recent effective model learning algorithms like \sift and \synth  learn from  pure \strips action traces,
and from  full state-action traces in \stripsp respectively, but can't deal with full or local observability
of selected state predicates. We have shown, however, that a suitable extension of \sift for learning mutex
features, called \siftp, in combination with  \synth, can address such tasks effectively, while 
providing  the theoretical conditions under which the algorithms are complete. For this, 
the query-synthesis algorithm of \synth yields queries or referring expressions, which \sift uses
to invent new predicates, in a loop that  generates  more queries and more predicates.
While \siftp learns  from action traces alone in \stripsp, and \synthp from partially observable
state-action traces, the only difference between the two algorithms is that the latter starts
with a non-empty set of observable predicates. Algorithms that initially  appeared to be very different,
\sift and \synth, are thus combined and extended in \synthp.
Future work includes  effective methods for learning  static predicates and for using
partial observability of the successor  states as well to further reduce the number
of action arguments that need to be observed.

\section*{Acknowledgments}
{We thank Jonas Reiter and Jakob Gebler for insightful comments and helpful discussions.}
The research has been supported by the Alexander von Humboldt Foundation with funds from the German Federal Ministry for Education and Research.
This project has received funding from the European Research Council (ERC) under the European Union's Horizon 2020 research and innovation programme (Grant agreement No. 885107).
This project was also funded by the German Federal Ministry of Education and Research (BMBF) and the Ministry of Culture and Science of the German State of North Rhine-Westphalia (MKW) under the Excellence Strategy of the Federal Government and the L\"ander.

\section*{AI Declaration}
The authors employed generative‑AI tools for proof‑reading and minor re‑phrasing (spell‑checking, grammar correction and stylistic improvements of the manuscript text),
and also  for code assistance (creation of function skeletons, program templates,  and debugging). All AI‑generated code was subsequently reviewed,
tested, and adapted by the authors. The authors retain full responsibility for  all the contents.

\bibliographystyle{kr}
\bibliography{control}

\end{document}